\documentclass[9pt,shortpaper,twoside,web]{ieeecolor}
\usepackage{generic}
\usepackage{cite}
\usepackage{amsmath,amssymb,amsfonts}
\usepackage{mathtools, cuted}
\usepackage{algorithmic}
\usepackage{graphicx}
\usepackage{textcomp}
\usepackage[utf8]{inputenc}
\usepackage{comment}
\usepackage{amssymb}

\usepackage{amsthm}
\usepackage{empheq}
\usepackage{hyperref}
\theoremstyle{plain}
\newtheorem{thm}{Theorem}[section]
\newtheorem{lem}[thm]{Lemma}

\newtheorem{cor}{Corollary}

\theoremstyle{defn}

\theoremstyle{rmk}
\newtheorem{rmk}{Remark}

\usepackage{amsmath}
\let\labelindent\relax
\usepackage{enumitem}
\def\BibTeX{{\rm B\kern-.05em{\sc i\kern-.025em b}\kern-.08em
    T\kern-.1667em\lower.7ex\hbox{E}\kern-.125emX}}
\begin{document}
\title{On the function approximation error for risk-sensitive reinforcement learning}
\author{Prasenjit Karmakar and Shalabh Bhatnagar
\thanks{ 
Prasenjit Karmakar is with the Department of Electrical Engineering, Technion, Israel (e-mail: k.prasenjit@campus.technion.ac.il, Contact No: +9720587706087). This research was conducted when Karmakar was a Ph.D. student in Indian Institute of Science Bangalore}
\thanks{Shalabh Bhatnagar is with the Department of Computer Science and Automation and Robert Bosch Centre for Cyber Physical Systems, IISC Bangalore (e-mail: shalabh@iisc.ac.in, Phone: 91(80) 2293-2987
Fax: 91(80) 2360-2911).}}

\maketitle

\begin{abstract}
We obtain several informative error bounds on function approximation for the policy evaluation algorithm proposed by Basu et al. for the risk-sensitive cost criteria represented using exponential utility. The novelty of our approach is that we use the irreducibility of a Markov chain (existing Bapat and Lindqvst inequality as well as the \textit{new} bound using Perron-Frobenius eigenvectors) to get the new bounds whereas the earlier work used spectral variation bound which holds for any matrix.
\end{abstract}


\section{Introduction}
\label{sec1}
The most familiar metrics in infinite horizon sequential decision problems are additive costs such as the  
discounted cost and the long-run average cost 
respectively. However, there is another cost criterion namely 
multiplicative cost (or risk-sensitive cost as it is better known) which has 
important connections with dynamic games and robust
control and is popular in certain applications, particularly related 
to finance where it offers the advantage of `penalizing all
moments', so to say, thus capturing the `risk' in addition to
mean return (hence the name).
For details see \cite{risk}. For a concrete example where such cost-criteria arise see the 
example of investor's portfolio in \cite[Section 3]{basu}.

Like other cost criteria, one can propose and justify
iterative algorithms for solving the dynamic programming
equation for the risk-sensitive setting \cite{bormn}. The issue we are interested in here is how to
do so, even approximately, when the exact model is either
unavailable or too unwieldy to afford analysis, but on
the other hand simulated or real data is available easily,
based on which one may hope to `learn' the solution in
an incremental fashion. 

One important point to note here is that the usual simulation based technique of calculating average cost  
does not work when the objective is a risk-sensitive cost. The reason is that average cost is defined as
\begin{align}
\lim_{n \to \infty}\frac{1}{n}E[\sum_{i=0}^{n-1}c(X_{i})], \nonumber 
\end{align}
where $c(i)$ is the cost of state $i$ and  $X_n$, $n\geq 0$ is 
an irreducible finite state Markov chain. Therefore the following iterative algorithm will almost surely converge to 
the average cost (under a given policy):
\begin{align}
\label{avg}
\theta_{n+1} = \theta_n +a(n)\left[c(X_n) - \theta_n\right], 
\end{align} where the step sizes $a(n), n\geq 0$ satisfy the Robbins-Monro conditions. 
This follows from the ergodic theorem for irreducible Markov chains as well as the convergence analysis of 
stochastic approximation with Markov noise \cite{borkar_markov}. On the contrary one needs to apply 
multiplicative ergodic theorem \cite{bmrisk} when the cost is risk-sensitive. However, this
does not have any closed-form limit. 
Moreover, one cannot even write iterative  algorithms like (\ref{avg}) in this setting because of the non-linear 
nature of the cost.
Due to the same 
reason, methods of \cite{marbach} also do not work in this setting when one is solving the full control 
problem.

This takes us into the domain of
reinforcement learning. In \cite{borkarq} and \cite{borkar_actor}, 
Q-learning and actor-critic methods have been proposed respectively for such a cost-criterion. 
These are `raw’ schemes in the sense that there is
no further approximation involved for the value function or policy. Since complex control
problems lead to dynamic programming equations in very
large dimensions (`curse of dimensionality'), one often
looks for an approximation based scheme. One such learning algorithm with function approximation is
proposed in \cite{basu}. 

In such approximation architectures an important problem is to obtain a good error bound 
for the approximation. The \textit{utility of such an error bound} is that the more one can improve the bound the more the 
approximated cost will be an accurate estimate of the actual cost. This has been pointed out by Basu et al. in the future 
work section of \cite{basu}, as well as by Borkar in the future work section of \cite{borkar_mont, borkar_conf}. The mathematical problem is to find a bound between largest eigenvalues of two matrices. While \cite{basu} provides such a bound
when the problem is policy evaluation, 
it is also mentioned there that the bound obtained 
is not good when the state space is large since the \textit{difference term} vanishes for large state space. In this technical note we investigate problems with the existing bound and  
then provide improved bounds. The main idea is to use Perron-Frobenius eigenvectors (they exist if we assume an irreducible Markov chain) to get the \textit{new bounds}, also to use the classical Bapat's inequality \cite{bapat}. One important novelty of the \textit{new} bound we provide compared to the existing Bapat and Lindqvist inequality \cite{henry} is that it is derived from an \textit{equality} which shows that 
the difference between the largest eigenvalues of two matrices $A$ and $B$ is the ratio of the inner products $\langle (A-B)x_A, x_{B^T}\rangle$ and $\langle x_A, x_{B^T}\rangle$. Here $x_A$ is the Perron-Frobenius 
eigenvector of $A$. More specifically, 
\begin{align}
\lambda - \mu &= \frac{\langle(A-B)x_A, x_{B^T}\rangle}{\langle x_A, x_{B^T}\rangle} =  \frac{\langle (A-B)x_{A^T}, x_{B}\rangle}{\langle x_{A^T}, x_{B}\rangle} \nonumber \\
&= \frac{\langle (A^T-B)x_A, x_{B}\rangle}{\langle x_A, x_{B}\rangle} = \frac{\langle(A-B^T)x_A, x_{B}\rangle}{\langle x_A, x_{B}\rangle}. \nonumber
\end{align}
Here $\lambda$ and $\mu$ are the largest eigenvalues of $A$ and $B$ matrices respectively. The novelty of our approach is that we make use of the irreducibility of Markov chain to get the new bounds whereas the earlier work by \cite{basu} used  spectral variation bound which is true for any matrix. Hence, all our bounds have \textit{difference terms} for large state space. We show that good approximations are captured using our bounds whereas the earlier bound would 
infer them as bad approximation.

The paper is organized as follows: Section \ref{sec2} describes the preliminaries and background to the problem 
considered. Section \ref{sec3} discusses the shortcomings of the bound proposed by \cite{basu}. Section \ref{sec4} describes the error bounds (\textit{new} error bound using Perron-Frobenius eigenvector and the error bounds using  the existing Bapat and Lindqvist inequality) as well as how 
they compare with each other and with the state of the art bound. In Section \ref{sec5} we 
provide an application of our results in the setting of the example of investor portfolio of \cite{basu}.
Section \ref{sec6} 
shows the theoretical conditions under which there is no error. This section also describes a temporal 
difference learning algorithm with its convergence analysis under the risk-sensitive setting. Section \ref{sec7} presents conclusions and some future 
research directions.
\section{Preliminaries and Background}
\label{sec2}
We begin by recalling the risk-sensitive framework. Consider an irreducible
aperiodic Markov chain $\{X_n\}$ on a finite state space $S = \{1, 2, \dots, s\}$, with transition matrix 
$P = [p(j|i)], i,j \in S$.
We are concerned here with the problem of prediction in controlled Markov chains i.e. the goal is to find the value of a given stationary policy and for which we shall consider a policy evaluation algorithm. Thus we have suppressed the explicit control dependence. 
Let $c:S \times S \to \mathbb{R}$ denote a prescribed `running cost' function and $C$ be the $s \times s$ matrix whose $(i,j)$-th 
entry is $e^{c(i,j)}$. The aim is to evaluate 
\begin{align}
\limsup_{n \to \infty} \frac{1}{n}\ln\left(E[e^{\sum_{m=0}^{n-1}c(X_m, X_{m+1})}]\right). \nonumber 
\end{align}That this limit exists follows from the
multiplicative ergodic theorem for Markov chains (see Theorem 1.2 of Balaji and Meyn (2000) \cite{bmrisk}, the
sufficient condition (4) therein is trivially verified for the finite state case here). Associated with
this is the multiplicative Poisson equation (see, e.g., Balaji and Meyn (2000) \cite[Theorem 1.2 (ii)]{bmrisk}):
We know from \cite{bmrisk} that there exists $\lambda >0$ and $V : S \to \mathbb{R^+}$ such that the  multiplicative Poisson 
equation holds as follows: 
\begin{align}
V(i) = \frac{\sum_j p(j|i)e^{c(i,j)}V(j)}{\lambda}. \nonumber 
\end{align}

An explicit expression for $V(\cdot)$ can be found from (5) in \cite{bmrisk} as in the following:
\begin{align}
V(i) = \lim_{n\to \infty} E\left[e^{\sum_{m=0}^{n-1}(c(X_m, X_{m+1}) -\ln(\lambda))} | X_0 =i\right],    
\end{align}
where $\lambda$ is the unique solution to 
\begin{align}
E[e^{\sum_{m=0}^{\tau_{i_0}-1}(c(X_m, X_{m+1})-\ln(\lambda))}|X_0 =i_0] =1.  \nonumber 
\end{align}
and $\tau_{i_0} = \min\{m>0: X_m =i_0\}$ for a prescribed $i_0 \in S$.
Thus $\lambda$ and  $V$ are respectively the Perron-Frobenius eigenvalue and eigenvector of the non-negative matrix 
$[[e^{c(i,j)}p(j|i)]]_{i,j \in S}$, whose existence is guaranteed by the Perron-Frobenius theorem. 
Furthermore,
under our irreducibility assumption, $V$ is specified uniquely up to a positive multiplicative scalar and $\lambda$ is
uniquely specified.
Also, the risk-sensitive 
cost defined as above
is $\ln\lambda$.

We know from \cite{bormn, borkarq,borkar_actor} that in the case of value iteration (with both dynamic programming 
and reinforcement learning), that the $i_0$-th component of the sequence of the iterates will converge to $\lambda$ where $i_0$ is a given fixed state in $S$. We consider a linear function approximation for the value function $V(i)$ where $V(i) \simeq \sum_{k=1}^M r^k\phi^k(i) = \phi^T(i) r$. Here $r = \left(r^1, \dots, r^M\right)^T$   is the vector of parameters $r^1, \dots, r^M$ and $\phi^k(\cdot), 1 \leq k \leq M$ are the basis functions or features chosen a priori. Also, $\phi(i) = \left(\phi^1(i), \dots, \phi^M(i)\right)^T$ denotes the feature of state $i$. Let $\Phi$ be an $s \times M$ matrix whose $(i,k)$-th entry is  $\phi^k(i)$ for $1 \leq i \leq s$ and $1 \leq k \leq M$. Let $I$ be the  $M \times M$ identity matrix and let $A_n$ and $B_n$, $n \geq 0$ be defined as 
\begin{align}
A_n = \sum_{m=0}^n e^{c(X_m, X_{m+1})}\phi(X_m)\phi^T(X_{m+1}),  \nonumber \\
B_n = \sum_{m=0}^n \phi(X_m)\phi^T(X_{m}), \nonumber
\end{align}
respectively. The linear function approximation version in \cite{basu} provides the following parameter update for $n \geq 0$:
\begin{align}
\label{lspe}
r_{n+1} = r_n + a(n)\left(\frac{B_n^{-1}A_n}{\max(\phi^T(i_0)r_n, \epsilon)}-I \right)r_n, 
\end{align}
where $\epsilon > 0$ is fixed (required to make the denominator of the first term inside the bracket non-zero). 
We also know from \cite[Theorem 5.3]{basu} that under the crucial assumption on the feature matrix that $\phi^k(\cdot), 1 \leq k \leq M$ are orthogonal vectors in the positive cone of $\mathcal{R}^s$ , and the submatrix of $P$ corresponding to $\cup_k \{i: \phi^k(i) > 0\}$ is irreducible (under this, $\Phi$ has the full rank $M$, see ($\dagger$) on p~883 in \cite{basu} for details)
the iterates 
$r_n$ 
satisfy the following: 
\begin{align}
\phi^T(i_0)r_n \to \mu,  \nonumber  
\end{align}
where $\mu >0$ is a Perron-Frobenius eigenvalue of the non-negative matrix $Q=\Pi \mathcal{M}$ with
$\Pi = \Phi(\Phi^T D \Phi)^{-1}\Phi^TD$ and $\mathcal{M} = C\circ P$ (unlike \cite{basu} we consider only a synchronous 
implementation for ease of understanding). Here
$D$ is a diagonal matrix with the $i$-th diagonal entry being $\pi_i$ where 
$\pi= \left(\pi_1, \pi_2, \dots, \pi_s\right)^T$ is the stationary distribution of $\{X_n\}$. 
Also, $e^{c(i,j)}p(j|i)$ is the $(i,j)$-th entry of $C \circ P$ where `$\circ$' denotes the component-wise product of two matrices with identical row and column dimensions. 
Assume that $\gamma_{ij}$ and $\delta_{ij}$ are the $(i,j)$-th entries of the matrices $C \circ P$ and $\Pi \mathcal{M}$ respectively. It is easy to check that 
\begin{align}
\label{gama}
\gamma_{ij} = e^{c(i,j)} p(j|i), \delta_{ij} = \frac{\phi^{k(i)}(i) \sum_{l =1}^{s} \phi ^{k(i)}(l) \pi_l \gamma_{lj}}{\sum_{m=1}^{s}{\phi^{k(i)}(m)}^2 \pi_m}
\end{align}

Therefore 
$\ln\mu$ serves as an approximation to the original risk-sensitive cost $\ln \lambda$. 
\textit{Our aim is to investigate the difference between these two i.e., $\ln(\frac{\lambda}{\mu})$, a quantity that plays the role of the error metric}.

\begin{rmk}
Throughout the paper the results are stated in general for matrices $A$ and $B$ with the largest eigenvalues of $A$ and $B$ as 
$\lambda >0$ 
and $\mu >0$ respectively. The entries of $A$ and $B$ should be clear from the context. 
\end{rmk}

\section{Related work and shortcomings}
\label{sec3}
Let $\|A\|$ be the operator norm of a matrix defined by $\|A\| =\inf\{c >0: \|Av\| \leq c\|v\| ~~\forall v\}$ where 
$\|v\|=\sum_{i=1}^s|v_i|$. Let $A = C\circ P$ and $B= \Pi \mathcal{M}$. 
The following bound was given in \cite{basu}: 
\begin{align}
\label{spect}
\ln\left(\frac{\lambda}{\mu}\right) \leq \ln \left(1+ \frac{(\|A\| + \|B\|)^{1-\frac{1}{s}} \|A - B\|^{\frac{1}{s}}}{\mu}\right),
\end{align}
using the spectral variation bound from \cite[Theorem VIII.1.1]{bhatia}, namely that if 
$A$ and $B$ are two $s \times s$ matrices with 
eigenvalues $\alpha_1, \dots, \alpha_s$ and $\beta_1, \dots, \beta_s$ respectively, then 
\begin{align}
\label{spect_bound}
\max_j \min_i |\alpha_i - \beta_j| \leq (\|A\| + \|B\|)^{1-\frac{1}{s}} (\|A-B\|)^{\frac{1}{s}}. 
\end{align}
This follows from the observation that if $\alpha_1 > 0$ and $\beta_1 >0$ 
are the leading eigenvalues of $A$ and $B$ respectively  
and $\alpha_1 \leq \beta_1$, then $|\alpha_1 - \beta_1| < \max_j \min_i |\alpha_i - \beta_j|$. Similar thing
happens for the case $\alpha_1 > \beta_1$ except that the roles of $\alpha_i$ and $\beta_j$ and hence the roles of 
$A$ and $B$ get reversed thus keeping the right hand side (R.H.S) of (\ref{spect_bound}) the same.

An important point to note is that when  $\alpha_1 \leq \beta_1$, the fact that $\beta_1$ is the leading eigenvalue of $B$ 
is not used. Same thing happens for the other case where $\alpha_1$ replaces $\beta_1$. 

Another important point above is that for large $s$ the bound given above cannot differentiate between the cases with two 
pairs of matrices $(A_1,B_1)$ and $(A_2,B_2)$ such that $\|A_1\| + \|B_1\| = \|A_2\| + \|B_2\|$ 
but $\|A_1 - B_1\|$ and $\|A_2 - B_2\|$ vary dramatically. 
This will be clear from the following simple example: Consider 
$A_1 = (x_{ij})_{s \times s}, B_1 = (y_{ij})_{s \times s}, A_2= (z_{ij})_{s \times s}, B_2 = (w_{ij})_{s \times s}$. Suppose 
$x_{ij} =p, y_{ij} =q, z_{ij} = p', w_{ij} = q'$ for $i,j = 1,2,\dots, s$, with $p + q = p' + q'$, $p'-q' > 0$ and 
$p-q > 0$. It is easy to see that $\|A_1\| + \|B_1\| = \|A_2\| + \|B_2\|$ and $r(A_1) = ps, r(B_1) = qs, r(A_2) = p's, r(B_2) = q's$.
Clearly, $p - q \neq p'-q'$ unless $pq = p'q'$. Here $r(A)$ denotes the Perron-Frobenius eigenvalue of matrix $A$.

In summary, when one is giving a bound between two quantities, the R.H.S 
should have terms involving the difference. However this 
does not occur while using spectral variation bound in the above example  as $(p-q)^{\frac{1}{s}}$ 
will converge to $1$ as $s \to \infty$. In Section \ref{sec5}, using the above example, we show that the new 
error bounds that we obtain contain always the \textit{difference terms} irrespective of the state space size $s$.
\newpage
\begin{strip}
\begin{empheq}[box=\fbox]{align}
       \ln\frac{\lambda}{\mu} &\leq \frac{1}{\lambda}\sum_{i,j=1}^s                                      e^{c(i,j)}p(j|i)x_i y_j \left[c(i,j) + \ln p(j|i)- 
       \ln\phi^{k(i)}(i) -\ln\left(\sum_{l=1}^s \phi^{k(i)}(l)\pi_la_{lj} \right)+ \ln\left(\sum_{m=1}^s{\phi^{k(i)}(m)}^2\pi_m\right)\right],        \label{first} \\
       \nonumber \\
 \ln\frac{\lambda}{\mu} &\leq \ln\left(\lambda\right) -\ln \left( \lambda - \sum_{i=1}^{s}x_i y_i \left(e^{c(i,i)}p(i|i)- \frac{\phi^{k(i)}(i)\sum_{l=1}^s \phi^{k(i)}(l)\pi_le^{c(l,i)}p(i|l)}{\sum_{m=1}^s{\phi^{k(i)}(m)}^2\pi_m}\right) - \right. \nonumber \\ 
 &\left. \sum_{i \neq j} e^{c(i,j)}p(j|i)x_i y_j \left(c(i,j) + \ln p(j|i)  + \ln\left(\frac{\sum_{m=1}^{s} {\phi^{k(i)}(m)}^2 \pi_m}{\sum_{l=1}^s \phi^{k(i)}(l) \pi_l e^{c(l,j)}p(j|l)}\right) - \ln \phi^{k(i)}(i) \right)\right), \label{second} \\
 \nonumber \\
 \ln\frac{\lambda}{\mu} &\leq \ln \left(1+ \frac{1}{\mu}\left(\sum_{i=1}^{s}x_i y_i \left(e^{c(i,i)}p(i|i)- \frac{\phi^{k(i)}(i)\sum_{l=1}^s \phi^{k(i)}(l)\pi_le^{c(l,i)}p(i|l)}{\sum_{m=1}^s{\phi^{k(i)}(m)}^2\pi_m}\right) + \right.\right. \nonumber \\ 
&\left. \left. \sum_{i \neq j} e^{c(i,j)}p(j|i)x_i y_j \left(c(i,j) + \ln p(j|i)  + \ln\left(\frac{\sum_{m=1}^{s} {\phi^{k(i)}(m)}^2 \pi_m}{\sum_{l=1}^s \phi^{k(i)}(l) \pi_l e^{c(l,j)}p(j|l)}\right)- \ln \phi^{k(i)}(i) \right)\right)\right), \label{third} \\
\nonumber \\
\frac{r(A)}{r(B)} &\leq  \Pi_{i,j=1} ^ s \left(\frac{a_{ij}}{b_{ij}}\right)^{\frac{a_{ij}x_i y_j}{r(A)}}, \label{ineq1} \\
\nonumber \\
\frac{r(A)}{r(B)} &\leq  \frac{r(A)}{r(A) + \sum_{i=1}^{s}x_i y_i (b_{ii}-a_{ii}) + \sum_{i \neq j} a_{ij}x_i y_j \ln\left(\frac{b_{ij}}{a_{ij}}\right)}, \label{ineq2} \\
\nonumber \\
\frac{r(A)}{r(B)} &\leq  1 + \frac{1}{r(B)}\left[\sum_{i=1}^{s}x_i y_i (a_{ii}-b_{ii}) + \sum_{i \neq j} a_{ij}x_i y_j \ln\left(\frac{a_{ij}}{b_{ij}}\right)\right], \label{ineq3} \\
\nonumber \\
&\min_i \sum_j e^{c(i,j)}p(j|i) > \sum_{i}x_i y_i\left(e^{c(i,i)}p(i|i)- \frac{\phi^{k(i)}(i)\sum_{l=1}^s \phi^{k(i)}(l)\pi_le^{c(l,i)}p(i|l)}{\sum_{m=1}^s{\phi^{k(i)}(m)}^2\pi_m}\right) + \sum_{i \neq j} x_i y_j e^{c(i,j)}p(j|i) \left(c(i,j) + \right.  \nonumber \\ 
& \left. \ln p(j|i)  +  \ln\left(\frac{\sum_{m=1}^{s} {\phi^{k(i)}(m)}^2 \pi_m}{\sum_{l=1}^s \phi^{k(i)}(l) \pi_l e^{c(l,j)}p(j|l)}\right)- \ln \phi^{k(i)}(i) \right), \label{as1} \\
\nonumber \\
& \forall i, e^{c(i,i)} p(i|i)\left(\sum_{m=1}^s {\phi^{k(i)}(m)}^2 \pi_m  - {\phi^{k(i)}(i)}^2 \pi_i\right) = \phi^{k(i)}(i) \sum_{l=1, l\neq i}^{s} \phi^{k(i)}(l) \pi_l a_{li}, \label{as2} \\
\nonumber \\
& \forall i \neq j, e^{c(i,j)} p(j|i)\left(\sum_{m=1}^s {\phi^{k(i)}(m)}^2 \pi_m  - {\phi^{k(i)}(i)}^2 \pi_i\right) = \phi^{k(i)}(i) \sum_{l=1, l\neq i}^{s} \phi^{k(i)}(l) \pi_l a_{lj}, \label{as3} \\
\nonumber \\
& \exists i \mbox{~~s.t~~} e^{c(i,i)} p(i|i)  > \max_{1 \leq l \leq s, l\neq i} e^{c(l,i)} p(i|l) \mbox{~~or,} \label{as5}\\
\nonumber \\
&\exists i \mbox{~~s.t~~} e^{c(i,i)} p(i|i)  < \min_{1 \leq l \leq s, l\neq i} e^{c(l,i)} p(i|l) \label{as6}
\end{empheq}
\end{strip}

\section{New error bounds}
\label{sec4}
In this section we present
examples where the bounds for $\ln \frac{\lambda}{\mu}$ using Bapat, Lindqvist inequality 
and  the
\textit{new} bound given in Section IV.B are much better
compared  to  the  spectral
variation  bound  in Section \ref{sec3}. We show  that  this  happens  due  to  the  presence
of \textit{difference term} in our bounds compared to spectral variation bound for
large state space.
\subsection{Bound based on Bapat and Lindqvist's inequality}
\label{bapat_ineq}
Recall the assumption ($\dagger$) on the feature matrix $\Phi$ from \cite{basu} which says that the feature matrix $\Phi$ has all non-negative entries 
and any two columns are orthogonal to each other. Here we strengthen the later part as follows: 

$(\star)$ Every row of the feature matrix $\Phi$ has exactly one positive entry i.e., for all $i$ there exists 
$1\leq k(i) \leq M$ such that $\phi^{j}(i) > 0$ if $j=k(i)$, otherwise $\phi^{j}(i)=0$.

Motivated by the discussion in Section \ref{sec3} and the fact that 
risk-sensitive cost is $\ln\lambda$ rather than $\lambda$ we need to find an 
upper bound on the error
$\ln \frac{\lambda}{\mu}$. Let $r(A)$ denote the Perron-Frobenius eigenvalue of the matrix $A =(a_{ij})_{s\times s}$. In the following we obtain three different bounds for the same quantity under the assumptions that 
a) $\lambda > \mu$, b) the matrix $P = [p(j|i)]$ has positive entries and impose conditions under
which one is better than the other.
Suppose $A$
admits left and right Perron eigenvectors $\textbf{x}, \textbf{y}$ respectively, with $\sum_i x_iy_i = 1$.
The three upper bounds of $\ln \lambda - \ln \mu$ are (\ref{first}) -(\ref{third}).

\begin{rmk}
\begin{enumerate}
 \item Note that in general it is hard to compare the bound given in  (\ref{spect}) 
 with the same in (\ref{first}) -(\ref{third}). We will only show that for the simple example of Section \ref{sec3} 
 the bounds given in (\ref{ineq1})-(\ref{ineq3}) are much better than the spectral variation bound when the state 
 space is large. Therefore $A$ and $B$ 
 will refer to matrices $A_1$ and $B_1$ respectively. It is easy to calculate
 $\|A\|, r(A), \mathbf{x}$ with this choice of $A$ and $B$. 
 Note that the actual error is $\ln (1 + \frac{\epsilon}{q})$ with $p = q+ \epsilon$ where $\epsilon << q$.
If we use (\ref{spect}) then the error is bounded by $\ln (3 + \frac{\epsilon}{q})$.
However, if we use (\ref{ineq1}) the error is bounded by $\ln (1 + \frac{\epsilon}{q})$ i.e. the actual error. If we use (\ref{ineq3}) the error 
is bounded by $\ln\left(1+\left(1+\frac{\epsilon}{q}\right) \ln \left(1 + \frac{\epsilon}{q}\right)\right)$ which 
reduces to $\ln\left(1 + \frac{\epsilon}{q}\right)$ (using $x+ x^2 \sim x \mbox{~~if~~} x<<1$). If we 
use (\ref{ineq2}) the error is bounded by $\ln\left(1 + \frac{\epsilon}{q}\right)$ (using the Binomial approximation theorem). 

\item If $A$ is such that all its diagonal elements are $p$ and the off-diagonal elements are $q$ then
for large state space the actual error is zero. If we use (\ref{ineq1}) then the bound is also zero whereas 
the right hand side of (\ref{spect}) is $\ln 3$. 
 
\item If $A$ is such that the entry in the first row and first column is $p$ and the rest are all $q$, then also a similar 
thing happens except the fact that now the right hand side of (\ref{spect}) is $\ln \left(1 + 2e^{-\frac{4q}{3}}\right)$.

\item Note that here $a_{ij} > b_{ij} \forall i,j \in \{1,2,\dots, s\}$ in the above example. Our bound will be much more 
useful in cases where there are $i,j$ such that $b_{ij} > a_{ij}$. From the definition of $\delta_{ij}, \gamma_{ij}$ 
it is clear that for all $j$ there  exists $i$ such that $\delta_{ij} > \gamma_{ij}$. In such a case, for 
every $j$ there will be at least one non-positive term inside the summation 
over $i$ which will make the bound small. The bound given 
in (\ref{spect}) does not capture such cases for large $s$.
\end{enumerate}
\end{rmk}

Here (\ref{second}) holds under (\ref{as1}) which follows from the fact that the following condition is necessary and 
sufficient for 
(\ref{ineq2}) to be true:
\begin{align}
\label{condn}
r(A) >  \sum_{i=1}^{s}x_i y_i (a_{ii}-b_{ii}) + \sum_{i \neq j} a_{ij}x_i y_j \ln\left(\frac{a_{ij}}{b_{ij}}\right)
\end{align}
and $\min_i \sum_j a_{ij} \leq r(A)$. Later in the proof of Lemma \ref{neq} we will see that,
in our setting, under $(\star)$, (\ref{condn}) gets
satisfied if the assumptions in Lemma \ref{neq} are true.

The bounds given in (\ref{ineq1})-(\ref{ineq3}) on $\frac{r(A)}{r(B)}$ immediately follow from the classic results of \cite[Theorem 1]{bapat} and \cite[Theorem 2]{henry}. In \cite[Theorem 3]{henry}, it 
is shown that under a condition on matrix entries, (\ref{first}) is better than (\ref{second}) whereas under another condition, the opposite holds. In the following we investigate how (\ref{third}) compares to the other two.

\begin{lem}
The bounds given in (\ref{third}) on $\ln\frac{\lambda}{\mu}$ are always better than (\ref{second}).
\end{lem}
\begin{proof}
Let $L = \sum_{i=1}^s x_i y_i (a_{ii} - b_{ii}) + \sum_{i,j=1, i \neq j}^s a_{ij} x_i y_j \ln\left(\frac{a_{ij}}{b_{ij}}\right)$. 
Now, from \cite[Theorem 2]{henry}, we know that
$L \geq r(A) - r(B)$ which implies that 
\begin{align}
L(L - r(A) + r(B)) \geq 0, \nonumber 
\end{align}
which further implies that 
\begin{align}
\frac{r(B) + L}{r(B)} \leq \frac{r(A)}{r(A) - L}. \nonumber
\end{align}
This means that the bound given in (\ref{third}) is better than (\ref{second}).
\end{proof}



\subsubsection{Some conditions}    
In this section we describe some sufficient conditions under which (\ref{first})-(\ref{second})
compare with each other. They will be referred in the next two lemmas.

\normalsize

\begin{lem}
\label{eq}
Assume that for all $i$, $b_{ii} = a_{ii}$ \cite[Theorem 3 (i)]{henry}. Then the bound given in (\ref{first}) is better than (\ref{second}). 
\end{lem}
\begin{proof}
Under the condition mentioned in  \cite[Theorem 3 (i)]{henry},
\begin{align}
 r(A) \Pi_{i \neq j} \left(\frac{b_{ij}}{a_{ij}}\right)^{\frac{a_{ij}x_i y_j}{r(A)}} \geq r(A) - L. \nonumber
\end{align}
Therefore (\ref{first}) is better than (\ref{second}). 
\end{proof}
\begin{rmk}
One such example where the condition of  Lemma \ref{eq} gets satisfied is:
$A= (a_{ij})_{s \times  s}$ with $a_{ij} = q$ if $i=j$ and $a_{ij}=p$ otherwise, 
  and $B= (b_{ij})_{s \times  s}$ with $b_{ij}=q$ for all $1 \leq i,j \leq s$ with $p-q \leq q$. 
It is easy to check that 
(\ref{condn}) gets satisfied for this example.
\end{rmk}  
\begin{rmk}
In our setting the condition mentioned in Lemma \ref{eq} gets satisfied if (\ref{as2}) is true. If the 
feature matrix is a single column matrix with all entries equal then  a sufficient condition for (\ref{as3}) 
is that for every $j$, $e^{c(i,j)}p(j|i)$ is the same for all $i$ (for example, the transition probabilities satisfy
$p(j|i) = e^{-c(i,j)}$ with the cost function $c(\cdot,\cdot)$ being non-negative).
\end{rmk}

\begin{lem}
\label{neq}
Assume that  for all $i \neq j$, $b_{ij} = a_{ij}$ and there is at least one $i$ 
such that $b_{ii} \neq a_{ii}$ \cite[Theorem 3 (ii)]{henry}. Then the bound given in (\ref{second}) is better than (\ref{first}).
\end{lem}

\begin{proof}
Under the condition mentioned in  \cite[Theorem 3 (ii)]{henry},
\begin{align}
 r(A) \Pi_{i=1}^s \left(\frac{b_{ii}}{a_{ii}}\right)^{\frac{a_{ii}x_i y_i}{r(A)}} \leq r(A) - L. \nonumber
\end{align}
Therefore (\ref{condn}) gets satisfied trivially if for all $i$, $b_{ii} \neq 0$ (which is true in our setting under 
$(\star)$ and condition $b)$ above). Therefore (\ref{second}) is better than (\ref{first}). 
\end{proof}
\begin{rmk}
In our setting the condition mentioned in Lemma \ref{neq} gets satisfied if (\ref{as3}) is true and there exists at least one 
$i$ for which either (\ref{as5}) or (\ref{as6}) are true (assuming that 
feature matrix is a single column matrix with all entries equal). If the 
feature matrix is a single column matrix with all entries equal then  a necessary and sufficient condition for (\ref{as5}) 
is that for every $j$, $e^{c(i,j)}p(j|i)$ is the same for all $i \neq j$. 
\end{rmk}

\subsection{\textbf{New} bound for non-negative matrices involving operator norm}
\label{op_norm}
Like Section \ref{bapat_ineq} here also we assume that $\lambda > \mu$. Note that the bounds derived in Section \ref{bapat_ineq} assume that the entries of $A$ are all positive. In this section we assume that the entries are only non-negative, however the matrix is irreducible.
If $A$ is an $s \times s$ normal matrix and $B$ is an arbitrary matrix then it is well-known \cite[Theorem VI.3.3]{bhatia} that 
\begin{align}
\label{normal}
|\lambda - \mu| \leq \|A-B\|.   
\end{align}
Now, for non-normal matrices it is not true. For example take the $s\times s$ matrix $A$ such that $a_{ij}= 1$ if $j=i+1$ and zero otherwise and $B$ is such that $b_{ij}=a_{ij}$ for all $i,j$ except that $b_{s1}=\epsilon$. Then the L.H.S of (\ref{normal}) becomes $\epsilon^{1/n}$ whereas the R.H.S becomes $\epsilon$. Therefore it is interesting to see whether one can give any bound involving $\|A-B\|$. In the following we provide such a bound. Assume that 
$\|v\|= \sum_{i=1}^s |v_i|$.
Let $\alpha(A) = \max_i (x_A)^{-1}_i$ where $x_A$ is the  Perron-Frobenius eigenvector of $A$
which has positive components if $A$ is irreducible. Let $x_B$ be the  Perron-Frobenius eigenvector of matrix $B$. Then we have the following result:
\begin{thm}
\label{mth}
\begin{align}
\label{invert}
 \ln\left(\frac{\lambda}{\mu}\right) \leq \ln\left(1+ \frac{\alpha(A^T) \|A- B\|)}{\mu}\right).
\end{align}
\end{thm}
\begin{proof}
\begin{align}
&\left(\lambda - \mu\right)\langle x_{A^T}, x_{B} \rangle \nonumber \\
&=\langle A^T x_{A^T}, x_{B}\rangle - \langle x_{A^T},B x_B\rangle \nonumber \\
&=\langle x_{A^T}, Ax_{B}\rangle - \langle x_{A^T},B x_B\rangle \nonumber \\
&= \langle  x_{A^T}, (A-B)x_B\rangle. \nonumber 
\end{align}
Moreover, 
\begin{align}
\langle x_{A^T}, x_{B} \rangle \geq {\alpha(A^T)}^{-1}\|x_{B}\| = {\alpha(A^T)}^{-1}. \nonumber
\end{align}
Then the proof follows from the observation that 
\begin{align}
|\langle  x_{A^T}, (A-B)x_B\rangle | \leq \|A-B\|. \nonumber
\end{align}
Here all the eigenvectors are normalized so that their norm is 1.

\end{proof}
\begin{rmk}
From the proof of Theorem \ref{mth} the following can be observed easily:
\begin{align}
\lambda - \mu &= \frac{\langle(A-B)x_A, x_{B^T}\rangle}{\langle x_A, x_{B^T}\rangle} =  \frac{\langle (A-B)x_{A^T}, x_{B}\rangle}{\langle x_{A^T}, x_{B}\rangle} \nonumber \\
&= \frac{\langle (A^T-B)x_A, x_{B}\rangle}{\langle x_A, x_{B}\rangle} = \frac{\langle(A-B^T)x_A, x_{B}\rangle}{\langle x_A, x_{B}\rangle} \nonumber
\end{align}
\end{rmk}
\begin{cor}
 If $A$ and $B$ are irreducible matrices with the minimum non-zero entry of $A$ being less than 1, then
 $|\lambda - \mu| \leq \max_i \frac{s (\max_i \sum_j a_{ij})^{s-1}\sum_j |a_{ij} - b_{ij}|}{(A_{\min})^{s-1}}$.
\end{cor}
\begin{proof}
From the above proof one can easily see that 
\begin{align}
\lambda - \mu = \frac{\langle  (A-B)x_A, x_{B^T}\rangle}{\langle x_A,  x_{B^T}\rangle}.
\end{align} Now, if we assume that $A$ and $B$ are both irreducible matrices then using the simple fact that if for all $i$, $q_i > 0$, then         $\frac{\sum_i p_i}{\sum_i q_i} \leq \max_i \frac{p_i}{q_i}$, we see that 
\begin{align}
\lambda - \mu &\leq \max_i \frac{((A-B)x_A)_i}{(x_A)_i} \\ \nonumber
              &\leq \max_i \frac{\sum_j (a_{ij} - b_{ij}) (x_A)_j}{(x_A)_i}
\end{align}
Now, $|\lambda - \mu| \leq \max_i \frac{s (\max_i \sum_j a_{ij})^{s-1}\sum_j |a_{ij} - b_{ij}|}{(A_{\min})^{s-1}}$. Note that in R.H.S, everything is in terms of matrix entries.
\end{proof}

\begin{cor}
For matrices $A$ whose column sums are equal, 
\begin{align}
 \ln\left(\frac{\lambda}{\mu}\right) \leq \ln\left(1+ \frac{s \|A- B\|)}{\mu}\right).\nonumber
\end{align}
\end{cor}
\begin{proof}
Follows from the fact that $\alpha(A^T) =s$.
\end{proof}

\begin{cor}
Equality condition is achieved in (\ref{invert}) iff $(A-B) x_B = \|A-B\| x_{A^T}$ and $\langle x_{A^T}, x_B \rangle = \alpha(A^T)^{-1}$.
\end{cor}
\begin{proof}
Follows trivially from the proof of (\ref{invert}).
\end{proof}
\begin{rmk}
Let us take $A = (a_{ij})_{s\times s}$ with $a_{ij} = p$ if $i=j$ and $a_{ij} =q$ otherwise and $b_{ij} = q$ for all $i,j$ with $p >q$.
Clearly for large $s$ the R.H.S of (\ref{invert}) becomes $\ln\left(\frac{p}{q}\right)$ whereas the R.H.S of (\ref{spect}) becomes $\ln 3$. Therefore if $p<3q$ (\ref{invert}) is a better bound than (\ref{spect}).   
\end{rmk}
\begin{rmk}
 As an illustrative example, consider the $s\times s$ non-symmetric matrix $A$ such that $a_{ij}= 1$ if $j=i+1$ and zero otherwise except that $a_{s1}=\epsilon_1$ and $B$ is such that $b_{ij}=a_{ij}$ for all $i,j$ except that $b_{s1}=\epsilon_2$. Also, assume that $\epsilon_1, \epsilon_2 \gg 1$ with $\epsilon_2 - \epsilon_1 \ll 1$. Clearly  for large $s$, $|\lambda - \mu|$ can be upper bounded by $s \left(\epsilon_2 - \epsilon_1 \right)$ using (\ref{invert}) whereas (\ref{spect}) becomes $\epsilon_1 + \epsilon_2$. 
\end{rmk}
\begin{rmk}
 Now assume that $\epsilon_1 < 1, \epsilon_2 >1$. Note that the smallest component $x_i$ in the Perron eigenvector is lower bounded by                      $\frac{(A_{\min})^{s-1}}{s (\max_i \sum_j a_{ij})^{s-1}}$ where $A_{\min}$ is the minimum non-zero element of $A$ and $\max_i \sum_j a_{ij} \geq 1$. The reason is that $x_i = \frac{\sum_j A_{ij}^{n} x_j}{\lambda^{n}}$ which is lower bounded by $\frac{A_{ij}^{n} x_j}{\lambda^{n}}$ where $\lambda$ is the largest eigenvalue of $A$ and $x_j$ is the largest component in the Perron eigenvector. Here $A^n_{ij} > 0$ where $n \leq s-1$. Then the R.H.S. of (\ref{invert}) is upper bounded by $\frac{s(\epsilon_2-\epsilon_1)}{\epsilon_1^{s-1}}$. 
\end{rmk}
The above two examples show the value of keeping \textit{difference terms} in the bound in case when the state space is large.
\begin{rmk}
\textbf{Improvement of the bound:} In case the minimum non-zero element in the $A$ matrix is less than 1 (this can be achieved if for some $(i,j)$, $c(i,j) <0$) and the largest eigenvalue of $B$ is greater than 0, using AM-GM inequality and lower bounding the $i$-th lowest component in the Perron eigenvector (using the fact that the highest component is lower bounded by $1/s$) one can improve the bound. 
We know that
\begin{align}
\langle x_A, x_B \rangle \geq \Pi_i (x_A)_i^{(x_B)_i}.
\end{align}
Now, $(x_B)_i$ can be upper bounded by $\frac{1}{\mu} \max B$ where $\max B$ is the maximum entry of $B$.
Now, for $1 \leq i \leq s-1$, the $i$-th lowest element in $x_A$ can be lower bounded by $\frac{(A_{\min})^{s-1}}{s(\max_i \sum_j a_{ij})^{s-1}}$ and the largest element in $x_A$ can be lower bounded by $\frac{1}{s}$.
Therefore,
\begin{align}
\langle x_A, x_B \rangle \geq \left(\frac{(A_{\min})^{(s-1)^2}}{s^s(\max_i \sum_j a_{ij})^{(s-1)^2}}\right)^{\frac{1}{\mu} \max B}.
\end{align}
\end{rmk}

\begin{rmk}
Suppose $A$ is the same as before while $b_{ij} = q'$ for all $i,j$
with $q'>q$.  It is easy to see that if the transition probability matrix is doubly stochastic then this kind of a situation arises if we assume that all the non-zero entries in the feature matrix are 1. Then (\ref{invert}) becomes an equality if $s = \frac{p-q}{q'-q}$.
\end{rmk}

\begin{rmk}
Note that $B$ need not be irreducible under the assumption $(\dagger)$ in \cite{basu}. Therefore, 
$x_B$ need not have all the components positive. 
\end{rmk}
\begin{rmk}
Similar bounds as above can be derived in the same way if $\lambda < \mu$. 
\end{rmk}
\section{Application of the results in the setting of the  example of investor portfolio in \cite{basu}}
\label{sec5}
In this section we revisit the example of investor portfolio in \cite{basu}. Let an investor’s portfolio consist only of $d$ stocks and one money market account. The
stock prices follow a $d$-dimensional ``geometric Brownian motion'' $s(t):= \left(s^{(1)}(t), \dots,s^{(d)}(t)\right)^T$ given by the
following stochastic differential equation:
\begin{align}
(\mbox{diag}(s(t)))^{-1}ds(t) = \tilde{b} dt + \tilde{\sigma} d\tilde{W}(t) \nonumber    
\end{align}
and the money market account $s^{(0)}(t)$ follows the equation:
\begin{align}
ds^{(0)}(t) = \bar{r}s^{(0)}(t)dt  \nonumber  
\end{align}
where $\tilde{W}(\cdot):=\left[\tilde{W}_1(\cdot), \dots, \tilde{W}_d(\cdot)\right]^T$ is a $d$-dimensional standard Brownian motion and the interest rate $\bar{r} \geq 0$, the drift vector $\tilde{b}=\left[\tilde{b}_1,\dots, \tilde{b}_d\right]^T \in \mathbb{R}^d$ , and the diffusion matrix
\begin{align}
\begin{pmatrix}
\tilde{\sigma}_{11} & \dots & \tilde{\sigma}_{1d}\\
\vdots & \vdots & \vdots \\
\tilde{\sigma}_{d1} & \dots & \tilde{\sigma}_{dd}
\end{pmatrix}
\end{align}
are assumed to be known. 
The investor is trying to optimize his ``returns'' under a ``safe policy'' of always keeping some money in his
money market account and distributing the rest over all assets in an optimal way. Specifically, he first ensures
that a fixed small amount $\delta > 0$ is taken from the total wealth (denoted by $V(\cdot)$) and put into the money
market account and then manipulates the fraction of his remaining total wealth in any particular asset $i$, given by $u^{(i)}(\cdot) = \frac{N_{i}(\cdot)s^{(i)}(\cdot)}{V(\cdot) - \delta}$. Here $N_{i}(t)$ = number of units of asset $i$ held at time $t$. Thus we can assume
that the investor’s control process $u(\cdot)$ takes values in $U=\{(u^{(0)}, u^{(1)}, \dots,u^{(d)})\in [0,1]^{d+1}: \sum_{i=0}^d u^{(i)}=1\}$. Under these assumptions, the ``remaining'' wealth process of the investor, defined as, $\hat{V}(t)=\sum_{i=0}^d N_i(t)s^{(i)}(t)=V(t) - \delta$, follows the s.d.e
\begin{align}
\frac{d\hat{V}(t)}{\hat{V}(t)} = \langle u(t), (\bar{r}, \tilde{b})\rangle dt + \langle \tilde{\sigma}'u^{[1:d]}(t), d\tilde{W}(t)\rangle \nonumber    
\end{align}
Thus the total wealth process $V(\cdot)$ follows the s.d.e.:
\begin{align}
\frac{dV(t)}{V(t)-\delta} = \langle u(t), (\bar{r},\tilde{b})\rangle dt + \langle \tilde{\sigma}'u^{[1:d]}(t),d\tilde{W}(t)\rangle    \nonumber
\end{align}
By a standard argument using the martingale representation theorem, we write the above equation for $V(t)$ as follows:
\begin{align}
\label{mrt}
dV(t) = b(V(t), u(t))dt + \sigma(V(t),u(t))dW(t)  
\end{align}
where $b(\cdot):\mathcal{R} \times U \to \mathcal{R}, \sigma(\cdot):\mathcal{R} \times U \to \mathcal{R}$
are given by
\begin{align}
b(V(t),u(t)) = (V(t)-\delta) \langle u(t), (\bar{r},\tilde{b})\rangle, \nonumber \\
\sigma(V(t),u(t)) = (V(t)-\delta) \parallel \tilde{\sigma}'u^{[1:d]}(t) \parallel \nonumber
\end{align}
and $W(\cdot)$ is a standard one-dimensional Brownian motion on a possibly enlarged probability space. Note that
starting with an initial wealth $V(0) = x > \delta$, it follows that $V(\cdot) \geq \delta$ %
a.s. for any control process $u(\cdot)$. We truncate and discretize the state space as $S_h =\{ \delta, \delta + h, \delta + 2h,\dots, \delta + Nh\}$. Given a fixed stationary
control process $u(\cdot)$ as per the above formulation and the corresponding portfolio process (\ref{mrt}) with initial value
$V_0 = x > 0$, we can exactly mimic the following arguments (for details see Examples 3 and 4 of \cite[p~ 95-98]{kush_dup})
to get an approximating Markov chain $\{V_{n}^{h}\}_{n=0,1,\dots}$ on the finite state space $S_{h}$:
Let $\tau = \min\{t:V(t)\notin (0,B), B >0\}$ and define the cost function $W(v,u) = E_v^u\left[\int_{0}^{\tau} k(v(s),u(s))ds + g(v(\tau))\right], W(v,u)=g(v),$ for $v=0,B$ where $k(v,u)$ is a bounded and continuous function. Formally applying Ito's formula to the
function $W(v, u)$ yields the equation
\begin{align}
L^{u}W(v,u) + k(v,u)=0, v \in (0,B) \nonumber    
\end{align}
with boundary conditions $W(0,u)=g(0),W(B,u)=g(B)$ where $L^u$ is
the differential operator of $V(\cdot)$ when the control is fixed at $u$. In particular,
\begin{align}
W_v(v,u)b(v,u) + W_{vv}(v,u)\sigma^2(v,u)/2 + k(v,u)=0,v\in (0,B) \nonumber   
\end{align}
Using the finite difference approximations
\begin{align}
f_v(v,u) \to \frac{f(v+h,u) -f(v,u)}{h}\mbox{~if~}b(v,u) \geq 0, \nonumber \\
f_v(v,u) \to \frac{f(v,u) -f(v-h,u)}{h}\mbox{~if~}b(v,u) < 0, \nonumber \\
f_{vv}(v,u) \to \frac{f(v+h,u) + f(v-h,u) -2f(v,u)}{h^2}\nonumber
\end{align}
we get the following approximating equation:
\begin{align}
W^h(v,u)=p^h(v,v+h|u)W^h(v+h,u) + \nonumber \\ 
p^h(v,v-h|u)W^h(v-h,u) +  k(v,u)\Delta^h(v,u) \nonumber
\end{align}
where we denote by $W^h(v,u)$ the finite difference approximation and the transition probabilities of the corresponding Markov chain can be written as 
\begin{align}
p^h(v,v+h|u) &= \frac{\frac{\sigma^2(v, u)}{2} +h b^+(v,u)}{\sigma^2(v) + h |b(v,u)|}, \nonumber\\
p^h(v,v-h|u) &= \frac{\frac{\sigma^2(v, u)}{2} +h b^-(v,u)}{\sigma^2(v) + h |b(v,u)|}, \nonumber\\
p^h(v,v'|u) &= 0 ~\forall v' \neq v \pm h \nonumber
\end{align}
and $\Delta^h(v,u)=\frac{h^2}{\sigma^2(v,u) + h|b(v,u)|}$.
With this Markov chain and the cost function $c(i,j)=\ln \frac{j}{i}$, one can easily calculate an estimate of the following  risk-sensitized expected exponential portfolio growth rate, also called the volatility-adjusted geometric mean return using algorithm (\ref{lspe}):
\begin{align}
J_{\theta}^h(v):=\liminf_{n \to \infty}\frac{-2}{\theta}\frac{1}{n}\ln E\left[e^{-\theta/2}\ln V_n^h|V_0^h=v\right] \nonumber    
\end{align}
where $\theta \in (0, \infty)$ can be interpreted as the risk-sensitivity parameter or risk-aversion parameter, because the
bigger its value the bigger the penalty associated with the portfolio’s risk.

Then one can easily construct the entries of $A$ and $B$ matrix (see (\ref{gama})) and use them in the results 
of Section \ref{sec4} to get the error bound for this specific example.
\section{Other improvements over \cite{basu}}
\label{sec6}
\subsection{Theoretical Conditions for zero error}

\subsubsection{Condition 1}
\begin{lem}
Let $\mathbf{x}$ be the left Perron eigen vector of the non-negative matrix $C \circ P$ i.e., $\mathbf{x}^T C \circ P = \lambda \mathbf{x}^T$.  
If $\Phi$ is an $s \times 1$ matrix  and  $\phi^i = y_i$ where $y_i = \frac{x_i}{\pi_i}$,
then $\mu = \lambda$, i.e., there will be no error when function approximation is deployed. 
\end{lem}
\begin{proof}
We know that $\delta_{ij} = \frac{\phi^{k(i)}(i) \sum_{l =1}^{s} \phi ^{k(i)}(l) \pi_l \gamma_{lj}}{\sum_{m=1}^{s}{\phi^{k(i)}(m)}^2 \pi_m}$, 
where $\gamma_{ij} = e^{c(i,j)} p(j|i)$.

We claim that with the choice of feature matrix as stated in the lemma, $\lambda$ is the eigenvalue of $B$ with eigenvector being $\mathbf{y}=(y_i)_{i \in \{1,2,\dots s\}}$.

\begin{align}
(\Pi \mathcal{M} y)_i = \sum_{k=1}^s \delta_{ik} y_k &= \sum_{k=1}^s \frac{y_i \sum_{l=1}^s x_l \gamma_{lk}}{\sum_{m=1}^s \frac{{x_m}^2}{\pi_m}} y_k \nonumber \\
&=\lambda y_i \frac{\sum_{k=1}^s \frac{{x_k}^2}{\pi_k}}{\sum_{m=1}^s \frac{{x_m}^2}{\pi_m}} \nonumber  
\end{align}

\end{proof}

\subsubsection{Condition 2}

Assume the ($\star$) from Section \ref{bapat_ineq}.






From \cite[Theorem 1]{bapat} it is easy to see that (this theorem is applicable due to Lemma 5.1 (ii) of \cite{basu} and $(\star)$)
the error can be zero even if 
$C\circ P \neq \Pi \mathcal{M}$, namely under the following conditions:
\begin{enumerate}
 \item there exists positive $\lambda_0, \beta_i, i=1,2, \dots, s$  such that 
 \begin{align}
  \delta_{ij} = \frac{\lambda_0 \gamma_{ij} \beta_i}{\beta_j}, i,j =1,2,\dots,s. \nonumber 
 \end{align}
\item $\Pi_{i,j=1}^s {\left(\delta_{ij}\right)}^{\gamma_{ij}x_i y_j} = \Pi_{i,j=1}^s {\left(\gamma_{ij}\right)}^{\gamma_{ij}x_i y_j}.$
\end{enumerate}

\begin{rmk}
Note that if the matrix $\Phi$ has a row $i$ with all $0$s, then $\delta_{ij} =0$ for all $j=1,2,\dots,s$ whereas $\gamma_{ij} >0$ for 
at least one $j \in \{1,2,\dots,s\}$ which violates the conditions for zero error stated above.
\end{rmk}

\subsection{Convergence of temporal difference learning algorithm in risk-sensitive setting}
%
The algorithm considered in (\ref{lspe}) involves matrix multiplication and matrix inverse. This problem can be simplified using
the temporal difference learning algorithm with function approximation for this setting as under:
\small
\begin{equation}
\label{rtd}
\theta_{n+1}=\theta_n + a(n)\left[\frac{e^{c(X_n, X_{n+1})} \phi^T(X_{n+1})\theta_n}{\phi^T(i_0)\theta_n} -\phi^T(X_n) \theta_n\right]\phi(X_n).
\end{equation}
\normalsize
The following theorem shows the convergence of recursion (\ref{rtd}).
\begin{thm}
Let $D$ be a diagonal matrix with first $d$ diagonal entries $\pi_i, 1 \leq i \leq d$ and the rest being zero. If $\Phi\Phi^T = D^{-1}$ and $\sup_n \|\theta_n\| < \infty \mbox{~a.s.}$ 
then $\phi^T(i_0)\theta_n \to \lambda_0$ as $n \to \infty$ where $\lambda_0$ is the largest eigenvalue of the leading $d \times d$ submatrix of $C \circ P$.
\end{thm}
\begin{proof}
First, we analyze the $\epsilon=0$ case.
Note that the algorithm tracks  the o.d.e
\begin{align}
\dot{\theta}(t) = \left(\frac{A'}{\phi^T(i_0)\theta(t)} - B'\right)\theta(t), \nonumber
\end{align}
where $A' = \Phi^TDC\circ P \Phi$ and $B' = \Phi^TD\Phi$.

This follows because it is easy to see that the algorithm tracks the o.d.e
\begin{align}
\dot{\theta}(t) = h(\theta(t)), \nonumber
\end{align}
where $h(\theta)= \sum_i \sum_j \pi(i)p(j|i)\left[\frac{\phi^T(j)\theta}{\phi^T(i_0)\theta} - \phi^T(i)\theta\right]\phi(i)$.

The above statement follows from the 
convergence theorem for stochastic recursive inclusion with Markov noise \cite{yaji_m} as 
the vector field in (\ref{rtd}) is merely continuous.

Now, the $k$-th entry of $A'\theta$ is 
\begin{align}
\langle \left(\sum_{i=1}^s \phi^k(i)  \sum_{j=1}^s e^{c(i,j)} p(j|i) \phi(j)\right), \theta \rangle.
\end{align}
Similarly, the $k$-th entry of $B'\theta$ can be shown to be the $k$-th entry of $\sum_i \sum_j \pi(i)p(j|i)\left[\phi^T(i)\theta\right]\phi(i)$.

Now, the claim follows directly from \cite[Theorem 5.3]{basu} (the synchronous implementation).
\end{proof}


\section{Conclusion}
\label{sec7}
In this technical note we gave several new informative bounds on the function approximation error for the policy evaluation procedure in the context of risk-sensitive reinforcement learning. An important future 
direction will be to design and analyze suitable learning algorithms to find 
the optimal policy with the accompanying error bounds. It will be interesting to see whether one can  
use our bounds for the policy evaluation problem to provide error bounds for the full control problem as well.
\bibliography{mybib}
\bibliographystyle{IEEEtran}
\end{document}